\documentclass[12pt]{msml2020}

\title[Neural network integral representations]{Neural network integral representations with the ReLU activation function}
\usepackage{times}

\msmlauthor{%
	\Name{Armenak Petrosyan}\Email{petrosyana@ornl.gov}\\
	\addr Oak Ridge National Laboratory
	\\
	\Name{Anton Dereventsov}\Email{dereventsov@gmail.com}\\
	\addr Oak Ridge National Laboratory
	\\
	\Name{Clayton~G. Webster}\Email{cwebst13@utk.edu}\\
	\addr University of Tennessee at Knoxville
	}

\usepackage[utf8]{inputenc} % allow utf-8 input
\usepackage[T1]{fontenc}	% use 8-bit T1 fonts
\usepackage{hyperref}	   % hyperlinks
\usepackage{url}			% simple URL typesetting
\usepackage{booktabs}	   % professional-quality tables
\usepackage{amsfonts}	   % blackboard math symbols
\usepackage{nicefrac}	   % compact symbols for 1/2, etc.
\usepackage{microtype}	  % microtypography
\usepackage{enumitem}
\usepackage{comment}
\usepackage{color}

\newcommand{\R}{\mathbb{R}}

\newcommand{\CC}{\mathbb{C}}
\newcommand{\Sd}{\mathbb{S}^d}

\newcommand{\dd}{\,\textrm d}		% d to use in integrals

\begin{document}

\maketitle

\begin{abstract}%
In this effort, we derive a formula for the integral representation of a shallow neural network with the ReLU activation function. 
We assume that the outer weighs admit a finite $L_1$-norm with respect to Lebesgue measure on the sphere.
For univariate target functions we further provide a closed-form formula for all possible representations.
Additionally, in this case our formula allows one to explicitly solve the least $L_1$-norm neural network representation for a given function.
\end{abstract}


\section{Introduction}%
In this paper we consider the problem of approximating a target function (e.g., an image classifier, solution to a partial differential equation, a specific parameter associated with a model, etc.) by a neural network.
The goal is to obtain, e.g., construct and train, a neural network that approximates the target function.

We propose to address this problem by the so-called ``integral representation'' technique.
The main ingredient of our method is to obtain a shallow network as an appropriate discretization of an integral representation of the objective function.
Specifically, we provide an approach to recast the $d$-dimensional function as an integral of a particular weight function $c : \mathbb{R}^d \times \mathbb{R} \to \mathbb{R}$ over the $d+1$ dimensional unit sphere. Based on the available training data, 
we approximate such integrals by a discrete sum, which, in turn, yields the desired network architecture.

Additionally, we introduce the space $\mathcal{W}(\R)$, that fully characterizes the class of functions that admit the desired integral representation.  Moreover, this allows us to solve the least $L_1$-norm network representation, i.e., the neural network with the minimal $L_1$-norm of the outer weights.
We note that the characterization of the multi-dimensional analogue of this space remains an open question; see, e.g., Remark~\ref{rem:W_multid} for more detailed information.

\subsection{Motivation}
Artificial neural networks were first introduced in the 1940's as mathematical models for describing biological networks, and with the invention of the back-propagation method for training neural networks~\cite{rumelhart1988learning} in the mid 1980's, the mathematical community's interest in this area spiked.
Though research in the field was slowly diminishing by the end of the 1990's, advancements in computational tools during the last decade have led to a revival of interest in this field, as deeper architectures have been observed to perform better than shallow ones~\cite{poggio2017and}, and faster GPU's accelerated the deep network training processes since, e.g., computations can be done in parallel using a network-type dataflow structure.

Nevertheless, despite recent developments in both theory and practical tools, many fundamental questions regarding the construction and training of (even shallow) neural networks still remain unanswered.
The unavailability of theoretical insight naturally implicates the numerous real-life challenges associated with the implementation and deployment of neural networks, are amplified.  These challenges include:
\begin{itemize}
	\item the choice of the network architecture is often dictated by a heuristic rather than the available data, which typically results in an underperforming or an over-complicated network;
	\item the model generally lacks interpretabilty in the sense that the contributions of individual nodes are generally unclear; and
	\item Backpropagation-based training is often more computationally expensive than necessary due to the network overparametrization and a sub-optimal initialization strategy. Moreover, the learning process is sensitive to the initial conditions, i.e., an initialization scheme and a choice of hyperparameters, and can result in a bad local minimum.
	%\item Obtained network is typically not robust in the sense that it is susceptible to adversarial attacks.
\end{itemize}

In this paper we attempt to tackle these issues by exploiting a more theoretical framework for understanding neural network approximations.
Specifically, we consider integral representations of shallow networks in order to analyze the relationship between the class of target functions and the corresponding tangible approximating networks.

\subsection{Integral representations of shallow neural networks}\label{sec:integral_neural_net}
A {\it shallow neural network} with an activation $\sigma : \R \to \R$ and $m$ nodes is a function $L : \R^d \to \R^{d^\prime}$ of the form 
\begin{equation}\label{eq:disc_NN}
	L(x) = \sum_{j=1}^m c_j \, \sigma(a_j \cdot x + b_j),
\end{equation}
where $a_j \in \R^d$, $b_j \in \R$, and $c_j\in \R^{d^\prime}$ are called the inner weights, biases, and outer weights, respectively.

In this paper we consider the ReLU (rectified linear unit) activation, given by $\sigma(z) = \max\{z,0\}$, which seems to be the conventional choice of activation in most modern architectures.
A neural network with the ReLU activation is a computationally simple parametric family since propagating an input through any such network essentially requires only matrix multiplications, which is a highly-optimized easily-parallelizable procedure.

Our main approach is to think of a shallow network $L : \R^d \to \R^{d^\prime}$ as a discretization of a suitable integral representation of the target function $f : \R^d \to \R^{d^\prime}$.
More precisely, a shallow neural network is regarded as a discretization of an integral of the form 
\begin{equation}\label{eq:int_formula_borel}
	f(x) = \int\limits_{\R^d\times \R} \sigma(a\cdot x+b) \dd\nu(a,b),
\end{equation}
where $\nu : \R^d \times \R \to \R^{d^\prime}$ is an appropriate Radon measure.
In particular, a network~\eqref{eq:disc_NN} with $m$  nodes  can be written as \eqref{eq:int_formula_borel} for an atomic measure $\nu$ with $m$ atoms so this type of representation is quite general.

We believe that such an approach opens many research opportunities, potentially leading up to faster and more stable algorithms for neural network training, and to architectures best fitted for specific problems. 
The integral form of neural network representations is more concise and better suited for analysis, thus allowing to address the questions of the architecture expressibility and the network approximation.
From numerical perspective, utilizing an appropriate discretization method allows one to obtain a fully-trained network that approximates the target function, potentially bypassing the backpropagation-based learning process.
Finally, an architecture obtained via the integral discretization can be treated as an initial state of the network, in place of the conventionally used random-based weights initialization, and hence can be further fine-tuned with an optimization algorithm.

The goal of this project is two-fold: first, we further develop existing analytical tools for neural network integral representations; second, we aim to facilitate the learning process by analyzing existing integral representations of neural networks and their integral discretizations.

\subsection{Related work}
Neural network integral representations have been considered by various authors, where typically the Radon measure $\nu$ is assumed to be of a special form, e.g. supported on a given set or absolutely continuous with respect to a probability measure.
One specific type of integral  representation for neural network integral representations  discussed below originates from the harmonic analysis perspective to shallow neural networks and employs the ridgelet transform; see e.g.~\cite{candes1999harmonic}.
There it is assumed that the target function $f$ can be written as
\begin{equation}\label{eq:dual_ridge_rep}
	f(x) = \int\limits_{\mathbb{R}^d \times \mathbb{R}} c(a,b) \, \sigma(a \cdot x+b) \dd a\dd b
	=: \mathcal{R}^\dagger_\sigma c(x).
\end{equation}
Function $\mathcal{R}^\dagger_\sigma c(x)$ is called the dual ridgelet transform of the function $c(a,b)$ with respect to $\sigma$.
The `direct' transform $\mathcal{R}_\tau f(a,b)$, called the ridgelet transform of $f(x)$ with respect to $\tau : \R \to \CC$, is given by
\begin{equation}\label{eq:ridge_rep}
	\mathcal{R}_\tau f(a,b) := \int\limits_{\mathbb{R}^d} f(x) \, \overline{\tau(a \cdot x + b)} \dd x.
\end{equation}
It is shown in~\cite{sonoda2017neural} that if the pair $(\sigma,\tau)$ satisfies the {\it admissibility condition}
\[
	(2\pi)^{d-1} \int\limits_\mathbb{R} \frac{\hat\sigma(\omega) \, \overline{\hat\tau(\omega)}}{|\omega|^d} \dd \omega = 1,
\]
then the {\it reconstruction formula} $\mathcal{R}^\dagger_\sigma \mathcal{R}_\tau (f) = f$ holds, thus providing a particular integral representation of the target function $f$:
\[
	f(x) = \int\limits_{\mathbb{R}^d \times \mathbb{R}} \int\limits_{\mathbb{R}^d} f(x) \, \overline{\tau(a \cdot x + b)} \, \sigma(a \cdot x + b) \dd x\dd a\dd b.
\]
In~\cite{ito1991representation}, using a Radon inversion formula, the author proves that for Heaviside and sigmoid-like activation functions, every objective function $f$ in the Schwartz class has a representation
\[
	f(x) = \int\limits_\R \int\limits_{\mathbb{S}^{d-1}} c(a,b) \, \phi(a \cdot x + b) \dd \nu(a) \dd b,
\]
where $\nu$ is some probability measure on the unit sphere in $\R^d$.
In~\cite{kuurkova1997estimates, kainen2000integral, kainen2010integral} the authors prove integral representation results of the form
\[
	f({x}) = \int\limits_{\R^{d+1}} c(a,b) \, \phi(a \cdot {x} + b) \dd a \dd b
\]
and use it to get error estimates for neural network approximations with Heaviside activation function.
Since Heaviside function is the derivative of ReLU this is highly relevant to our work, and we employ some of their results in the proof of Theorem~\ref{thm:main_multi_inv}.

Largely motivated by the works~\cite{barron1993universal} and~\cite{klusowski2016risk}, in~\cite{ma2019barron} the Barron spaces were introduced, which are defined as the space of functions $f : \R^d \to \R$ admitting the representation
\begin{equation*}\label{eq:barron_cond}
	f(x) = \int\limits_{\R \times \R^d \times \R} c \, \phi(a \cdot x + b) \dd \nu(a,b,c)
	\ \ \text{for all}\ \ 
	x \in [-1,1]^d,
\end{equation*}
where $\nu$ denotes the space of probability measures on $\R \times \R^d \times \R$.
Note that in a Barron space the representation is restricted to the unit cube $[-1,1]^d$, whereas we require the representation to be on the whole $\R^d$.

While the above integral representations are largely considered with the aim of obtaining estimates on the size of the approximating network, methods for discretizing integral representations of the form~\eqref{eq:int_formula_borel} have been considered by various authors aiming to obtain desirable approximation rate.
In particular, in~\cite{bengio2006convex} the authors employ a greedy method to discretize the solution of~\eqref{eq:int_formula_borel} with the smallest total variation norm, and in~\cite{bach2017breaking} the same problem is solved by the conditional gradient algorithm.
In~\cite{pao1994learning,pao1995functional} the authors suggested the random vector functional-link (or RVFL) network method, which includes Monte--Carlo sampling for the values of the parameters $(a,b)$ and least square regularization for the values of outer weights $c$.
A related Monte--Carlo discretization method for integral representations of radial basis function (RBF) neural networks is considered in~\cite{mhaskar2004tractability}.
Lastly, in~\cite{sonoda2013nonparametric} integral representations are used to get better weight initialization.

The major difference of our approach is that we are aiming to fully characterize the class of functions that allow neural network integral representation, and to use that representation to get a meaningful interpretation to network weights (e.g., dependence on second derivative in one dimensional case as stated in Theorem~\ref{thm:main_inv}).

\subsection{Our approach}
While the representation~\eqref{eq:int_formula_borel} can generally be stated for a wider class of  measures $\nu$, in this paper we restrict ourselves to the case of Lebesgue measures, which seems appropriate from a harmonic analysis viewpoint.

Due to the positive homogeneity of the ReLU activation function the representation~\eqref{eq:disc_NN} can be rewritten with the weights $(a_j,b_j)$ on the unit sphere $\Sd = \{(a,b)\in\R^{d+1} : \|a_j\|^2 + |b_j|^2 = 1\}$.
In this setting, we consider target functions $f$ that admit integral representations of the form
\begin{equation}\label{form:sphere_rep}\tag{$\star$}
	f(x) = \int\limits_{\Sd} c(a,b) \, \sigma(a \cdot x + b) \dd\nu_d(a, b),
\end{equation}
where $\nu_d$ is the Lebesgue measure on $\Sd$ and $c\in L_1(\Sd,\nu_d)$ with $L_1(\Sd,\nu_d)$ denoting the class of all Lebesgue-integrable functions on $\Sd$ with respect to $\nu_d$.
Note that the integral in~\eqref{form:sphere_rep} is well defined since for every $x\in\R^d$ the function $\sigma(a \cdot x + b) $ is bounded on $\Sd$.

In this paper we address the following challenges, which we fully solve in a case of univariate target functions:
\begin{enumerate}
	\item[(a)] Characterize the class of target functions that admit the representation~\eqref{form:sphere_rep};
	\item[(b)] For a given $f$ find all weight functions $c \in L_1(\Sd,\nu_d)$ for which $\eqref{form:sphere_rep}$ holds;
	\item[(c)] Find the least $L_1$-norm solution to $\eqref{form:sphere_rep}$ for a given target function.
\end{enumerate}
We recognize that a similar approach is employed in~\cite{pmlr-v99-savarese19a, 2019arXiv191001635O}, however, to the best of our knowledge, the characterization results in Section~\ref{subsec:characterization} involving the space $\mathcal{W}(\R)$ are novel and presented in this effort for the first time.

\section{Main results}
We begin this section by recalling the following well-known definitions.
The Radon transform of a function $f\in L^1(\R^d)$ is a mapping $\mathcal{R}[f] : \R^{d+1} \to \R^d$ given by the formula:
\[
	\mathcal{R}[f](a,b) := \int\limits_{a \cdot x + b = 0} f(x) \dd\nu_{d-1},
\]
where integration is with respect to $(d-1)$-dimensional Lebesgue measure $\nu_{d-1}$ on the hyperplane $\{x \in \R^{d} : a \cdot x + b = 0\}$.
The Hilbert transform $\mathcal{H : \R \to \R}$ of a function $g : \R \to \R$ is defined as
\[
	\mathcal{H}[g](b) := \frac{1}{\pi} \text{ p.v.} \int\limits_{-\infty}^\infty \frac{g(z)}{b-z} \dd z.
\]
We now formulate one of our main results which provides a particular weight function $c(a,b)$ for the integral representation~\eqref{form:sphere_rep}.
\begin{theorem}\label{thm:main_multi_inv}
For any compactly supported function $f\in C^{d+1}(\R^d)$ define
\[
	c_f(a,b) = \left\{\begin{array}{ll}
		\displaystyle{\frac{(-1)^{\nicefrac{d+1}{2}}}{2(2\pi)^{d-1}} \frac{1}{\|a\|^{d+2}} \,
		\frac{\partial^{d+1}}{\partial b^{d+1}} \mathcal{R}[f](a,b) \, \sigma(a \cdot x + b)}
		&\quad \text{if } d \text{ is odd};
		\medskip\\
		\displaystyle{\frac{(-1)^{\nicefrac{d}{2}}}{2(2\pi)^{d-1}} \frac{1}{\|a\|^{d+2}} \,
		\frac{\partial^{d+1}}{\partial b^{d+1}} \mathcal{H}\big[ \mathcal{R}[f](a,b) \big](b) \, \sigma(a \cdot x + b)}
		&\quad \text{if } d \text{ is even}.
	\end{array}\right.
\]
Then we have
\[
	f(x) = \int\limits_{\Sd} c_f(a,b) \dd \nu_d(a,b).
\]
\end{theorem}

The stated theorem offers a way to construct a specific weight function $c(a,b)$, for which the integral representation~\eqref{form:sphere_rep} holds.
A similar result was proved in~\cite{kainen2010integral} for the Heaviside function, which is the derivative of ReLU.
In the next section we show that for the univariate target functions the particular $c_f(a,b)$ provided by Theorem~\ref{thm:main_multi_inv} has the least $L_1$-norm among all possible solutions $c\in L_1(\Sd,\nu_d)$.
While previously we conjectured that this is likely to be the case for any dimension $d > 1$ as well, it is now evident that this conjecture in fact holds and the proof can be derived (after a small adaptation to our setting) from~\cite{2019arXiv191001635O}, that was posted concurrently with our work.

\subsection{Univariate target functions}\label{subsec:characterization}
For the case $d = 1$ we state a stronger version of Theorem~\ref{thm:main_multi_inv} that characterizes the class of target functions $f(x)$ that can be represented in the form \eqref{form:sphere_rep}.

Note that the unit circle $\mathbb{S}^1$ can be parameterized by $(a,b) = (\cos\phi,\sin\phi)$, where $\phi\in[0,2\pi)$.
Then the hyperplane $\{x \in \R : a \cdot x + b = 0\}$ consists of a single point: $x = -\tan\phi $ and hence $\mathcal{R}[f](a,b)=f(-\tan\phi )$. 
Thus Theorem~\ref{thm:main_multi_inv} provides for any compactly supported function $f\in C^2(\R)$ that
\[
	f(x) = \int\limits_{0}^{2\pi} \frac{f^{\prime\prime}(-\tan\phi )}{2|\cos^3\phi|} \,\sigma(x\cos\phi + \sin\phi) \dd\phi.
\]
In this subsection we provide a more general representation and extend the set of the admissible functions $f$ by defining the class $\mathcal{W}(\mathbb{R}) \supset C^2(\R)$ consisting of such functions $g : \R \to \R$ that $g^\prime$ exists everywhere on $\R$, $g^{\prime\prime}$ exist almost everywhere on $\R$, and
\[
	\lim\limits_{x \to \pm\infty} g(x) = \lim\limits_{x \to \pm\infty} x g^\prime(x) = 0,
	\quad
	\int\limits_{-\infty}^\infty |g^{\prime\prime}(x)| \sqrt{1 + x^2} \dd x < \infty.
\]

The following theorem characterizes the class of target functions that admits the integral representation~\eqref{form:sphere_rep} with an integrable weight function $c(a,b)$.
Moreover, for a given target function $f$ we characterize the class of integrable kernels $c$ that satisfy the representation~\eqref{form:sphere_rep}.
\begin{theorem}\label{thm:main_inv}
The function $f$ admits the representation
\begin{equation}\label{form:circ_rep}
	f(x) = \int\limits_{0}^{2\pi} c(\phi) \, \sigma(x\cos\phi + \sin\phi ) \dd\phi
\end{equation}
with some $c \in L_1[0,2\pi]$ if and only if $f$ has the form
\begin{equation}\label{form:Wrep}
	f(x) = g(x) + \alpha x + \beta + \gamma(x \arctan x + 1) + \eta\arctan x,
\end{equation}
where $g \in \mathcal{W}(\mathbb{R})$ and
\begin{gather*}
	\alpha = \frac{2}{\pi}\int\limits_{0}^{2\pi} c(\phi) \cos\phi  \dd\phi,
	\quad
	\beta = \frac{2}{\pi}\int\limits_{0}^{2\pi} c(\phi) \sin\phi  \dd\phi,
	\quad
	\gamma = \int\limits_{0}^{2\pi} c(\phi) \, |\cos\phi| \dd\phi,
	\\
	\eta = \int\limits_{0}^{2\pi} c(\phi) \, s(\phi) \dd\phi
	\ \text{ with }\ 
	s(\phi) = s(\phi+\pi) = \sin\phi 
	\ \text{ for }\ 
	\phi \in [\nicefrac{-\pi}{2}, \nicefrac{\pi}{2}).
\end{gather*}
Moreover, the set of such weight functions $c(\phi)$ coincides with the set of functions of the form
\begin{equation}\label{eq:weight_function_form}
	\frac{f^{\prime\prime}(-\tan\phi )}{2|\cos^3\phi|}
	+ \frac{k^{\prime\prime}(-\tan\phi )}{2\cos^3\phi}
	+ \alpha \cos\phi  + \beta \sin\phi + \gamma \, |\cos\phi| + \eta \, s(\phi),
\end{equation}
where $k \in \mathcal{W}(\mathbb{R})$ and $\alpha, \beta, \gamma, \eta \in \mathbb{R}$.
\end{theorem}

The following theorem characterizes the class $\mathcal{W}(\R)$ as the functions that admit an integral representation with an appropriate integrable weight function.
\begin{theorem}\label{thm:WR_rep}
A function $f$ belongs to the class $\mathcal{W}(\R)$ if and only if it admits the representation
\[
	f(x) = \int\limits_{0}^{2\pi} c(\phi) \, \sigma(x \cos\phi + \sin\phi ) \dd\phi
\]
with a weight function $c \in L_1[0,2\pi]$ satisfying
\[
	\int\limits_{-\nicefrac{\pi}{2}}^{\nicefrac{\pi}{2}} c(\phi) \, \cos\phi \dd\phi
	= \int\limits_{-\nicefrac{\pi}{2}}^{\nicefrac{\pi}{2}} c(\phi+\pi) \, \cos\phi \dd\phi
	= \int\limits_{-\nicefrac{\pi}{2}}^{\nicefrac{\pi}{2}} c(\phi) \, \sin\phi \dd\phi
	= \int\limits_{-\nicefrac{\pi}{2}}^{\nicefrac{\pi}{2}} c(\phi+\pi) \, \sin\phi \dd\phi
	= 0.
\]
\end{theorem}

Since in Theorem~\ref{thm:WR_rep} we make an assumption $c \in L_1[0,2\pi]$, we can pose a question of finding the weight function $c(a,b)$ with the smallest $L_1$-norm for a given target function $f$.
Such a formulation is of interest for many real-life applications as regularization is typically employed to condition ill-posed problems (see, e.g., \cite{engl1996regularization, evgeniou2000regularization}).
In particular, $L_1$-norm minimization is commonly used in compressed sensing for finding a sparse solution, and is often utilized in machine learning for promoting generalization properties of the network.
In the following theorem we answer the stated question.

\begin{theorem}\label{thm:least_l1}
For $f \in \mathcal{W}(\R)$ the minimum
\begin{equation}\label{eq:min_c_1_norm}
	\min\limits_{c\in L_1[0,2\pi]} \|c\|_1
	\quad\text{s.t.}\quad
	\int\limits_{0}^{2\pi} c(\phi) \, \sigma(x \cos\phi + \sin\phi ) \dd\phi = f(x)
\end{equation}
is attained at 
\[
	c_f(\phi) = \frac{f^{\prime\prime}(-\tan\phi )}{2|\cos^3\phi|}.
\]
\end{theorem}
\begin{remark}
We note that the solution of~\eqref{eq:min_c_1_norm} is not always unique.
For instance, if $f^{\prime\prime}(x) > 0$ a.e. then any $g \in \mathcal{W}(\R)$ with $|g^{\prime\prime}| \le f^{\prime\prime}$ provides a weight function with the smallest possible norm. 
\end{remark}

\begin{remark}\label{rem:W_multid}
A similar to Theorem~\ref{thm:least_l1} result was obtained in~\cite{pmlr-v99-savarese19a} and its multidimensional analogue in~\cite{2019arXiv191001635O}.
Their results are stated in the $\mathbb{S}^{d-1} \times \R$ domain of $(a,b)$, which corresponds to a different scaling of the weights.
After performing respective change of variables from $\mathbb{S}^{d-1} \times \R$ to $\Sd$, one of the theorems in~\cite{2019arXiv191001635O} implies that in fact the weight function $c_f$ provided by Theorem~\ref{thm:main_multi_inv} indeed possesses the smallest $L_1$-norm in any setting $d > 1$.
However the question of finding the analogue of the space $\mathcal{W}(\R)$ in multiple dimensions remains open still.
\end{remark}

\section{Conclusion}
This effort focused on integral representations of shallow neural networks with ReLU activation functions.
Specifically, we recast a target function in a suitable integral form, which can be discretized in order to obtain a network approximation of the training data.

We analyze the set of target functions that admit the desired integral form and derive an explicit formula for the integrand.
Moreover, in the univariate setting, we fully characterize all such functions as the class $\mathcal{W}(\mathbb{R})$, and establish an approach for obtaining a network with the least $L_1$-norm of the outer weights, for any function from $\mathcal{W}(\mathbb{R})$.

Our approach facilitates a ``deeper" theoretical understanding of how the network weights' contribute to the approximation of the training data. We believe that it is vital to bridge the gap between practical applications and underlying theoretical processes, and hope that the tools presented in this work will contribute to providing solutions to this grand challenge.
We intend to continue research in this direction and further promote the interpretability of neural networks by better understanding how the geometry of the training data affects the architecture and the training process of the neural network.

\acks{This material is based upon work supported in part by: the U.S. Department of Energy, Office of Science, Early Career Research Program under award number ERKJ314; U.S. Department of Energy, Office of Advanced Scientific Computing Research under award numbers ERKJ331 and ERKJ345; the National Science Foundation, Division of Mathematical Sciences, Computational Mathematics program under contract number DMS1620280; and by the Laboratory Directed Research and Development program at the Oak Ridge National Laboratory, which is operated by UT-Battelle, LLC., for the U.S. Department of Energy under contract DE-AC05-00OR22725.}

\bibliography{NNrefs}

\begin{thebibliography}{21}
\providecommand{\natexlab}[1]{#1}
\providecommand{\url}[1]{\texttt{#1}}
\expandafter\ifx\csname urlstyle\endcsname\relax
  \providecommand{\doi}[1]{doi: #1}\else
  \providecommand{\doi}{doi: \begingroup \urlstyle{rm}\Url}\fi

\bibitem[Bach(2017)]{bach2017breaking}
Francis Bach.
\newblock Breaking the curse of dimensionality with convex neural networks.
\newblock \emph{The Journal of Machine Learning Research}, 18\penalty0
  (1):\penalty0 629--681, 2017.

\bibitem[Barron(1993)]{barron1993universal}
Andrew~R Barron.
\newblock Universal approximation bounds for superpositions of a sigmoidal
  function.
\newblock \emph{IEEE Transactions on Information theory}, 39\penalty0
  (3):\penalty0 930--945, 1993.

\bibitem[Bengio et~al.(2006)Bengio, Roux, Vincent, Delalleau, and
  Marcotte]{bengio2006convex}
Yoshua Bengio, Nicolas~L Roux, Pascal Vincent, Olivier Delalleau, and Patrice
  Marcotte.
\newblock Convex neural networks.
\newblock In \emph{Advances in neural information processing systems}, pages
  123--130, 2006.

\bibitem[Cand{\`e}s(1999)]{candes1999harmonic}
Emmanuel~J Cand{\`e}s.
\newblock Harmonic analysis of neural networks.
\newblock \emph{Applied and Computational Harmonic Analysis}, 6\penalty0
  (2):\penalty0 197--218, 1999.

\bibitem[Engl et~al.(1996)Engl, Hanke, and Neubauer]{engl1996regularization}
Heinz~Werner Engl, Martin Hanke, and Andreas Neubauer.
\newblock \emph{Regularization of inverse problems}, volume 375.
\newblock Springer Science \& Business Media, 1996.

\bibitem[Evgeniou et~al.(2000)Evgeniou, Pontil, and
  Poggio]{evgeniou2000regularization}
Theodoros Evgeniou, Massimiliano Pontil, and Tomaso Poggio.
\newblock Regularization networks and support vector machines.
\newblock \emph{Advances in computational mathematics}, 13\penalty0
  (1):\penalty0 1, 2000.

\bibitem[Ito(1991)]{ito1991representation}
Yoshifusa Ito.
\newblock Representation of functions by superpositions of a step or sigmoid
  function and their applications to neural network theory.
\newblock \emph{Neural Networks}, 4\penalty0 (3):\penalty0 385--394, 1991.

\bibitem[Kainen et~al.(2000)Kainen, K\r{u}rkov{\'a}, and
  Vogt]{kainen2000integral}
Paul~C Kainen, V\v{e}ra K\r{u}rkov{\'a}, and Andrew Vogt.
\newblock An integral formula for heaviside neural networks.
\newblock \emph{Neural Network World}, 10:\penalty0 313--319, 2000.

\bibitem[Kainen et~al.(2010)Kainen, K\r{u}rkov{\'a}, and
  Vogt]{kainen2010integral}
Paul~C Kainen, V\v{e}ra K\r{u}rkov{\'a}, and Andrew Vogt.
\newblock Integral combinations of heavisides.
\newblock \emph{Mathematische Nachrichten}, 283\penalty0 (6):\penalty0
  854--878, 2010.

\bibitem[Klusowski and Barron(2016)]{klusowski2016risk}
Jason~M Klusowski and Andrew~R Barron.
\newblock Risk bounds for high-dimensional ridge function combinations
  including neural networks.
\newblock \emph{arXiv preprint arXiv:1607.01434}, 2016.

\bibitem[K\r{u}rkov{\'a} et~al.(1997)K\r{u}rkov{\'a}, Kainen, and
  Kreinovich]{kuurkova1997estimates}
V\v{e}ra K\r{u}rkov{\'a}, Paul~C Kainen, and Vladik Kreinovich.
\newblock Estimates of the number of hidden units and variation with respect to
  half-spaces.
\newblock \emph{Neural Networks}, 10\penalty0 (6):\penalty0 1061--1068, 1997.

\bibitem[Ma et~al.(2019)Ma, Wu, et~al.]{ma2019barron}
Chao Ma, Lei Wu, et~al.
\newblock Barron spaces and the compositional function spaces for neural
  network models.
\newblock \emph{arXiv preprint arXiv:1906.08039}, 2019.

\bibitem[Mhaskar(2004)]{mhaskar2004tractability}
Hrushikesh~Narhar Mhaskar.
\newblock On the tractability of multivariate integration and approximation by
  neural networks.
\newblock \emph{Journal of Complexity}, 20\penalty0 (4):\penalty0 561--590,
  2004.

\bibitem[{Ongie} et~al.(2019){Ongie}, {Willett}, {Soudry}, and
  {Srebro}]{2019arXiv191001635O}
Greg {Ongie}, Rebecca {Willett}, Daniel {Soudry}, and Nathan {Srebro}.
\newblock {A Function Space View of Bounded Norm Infinite Width ReLU Nets: The
  Multivariate Case}.
\newblock \emph{arXiv e-prints}, art. arXiv:1910.01635, Oct 2019.

\bibitem[Pao and Phillips(1995)]{pao1995functional}
Yoh-Han Pao and Stephen~M Phillips.
\newblock The functional link net and learning optimal control.
\newblock \emph{Neurocomputing}, 9\penalty0 (2):\penalty0 149--164, 1995.

\bibitem[Pao et~al.(1994)Pao, Park, and Sobajic]{pao1994learning}
Yoh-Han Pao, Gwang-Hoon Park, and Dejan~J Sobajic.
\newblock Learning and generalization characteristics of the random vector
  functional-link net.
\newblock \emph{Neurocomputing}, 6\penalty0 (2):\penalty0 163--180, 1994.

\bibitem[Poggio et~al.(2017)Poggio, Mhaskar, Rosasco, Miranda, and
  Liao]{poggio2017and}
Tomaso Poggio, Hrushikesh Mhaskar, Lorenzo Rosasco, Brando Miranda, and Qianli
  Liao.
\newblock Why and when can deep-but not shallow-networks avoid the curse of
  dimensionality: a review.
\newblock \emph{International Journal of Automation and Computing}, 14\penalty0
  (5):\penalty0 503--519, 2017.

\bibitem[Rumelhart et~al.(1988)Rumelhart, Hinton, Williams,
  et~al.]{rumelhart1988learning}
David~E Rumelhart, Geoffrey~E Hinton, Ronald~J Williams, et~al.
\newblock Learning representations by back-propagating errors.
\newblock \emph{Cognitive modeling}, 5\penalty0 (3):\penalty0 1, 1988.

\bibitem[Savarese et~al.(2019)Savarese, Evron, Soudry, and
  Srebro]{pmlr-v99-savarese19a}
Pedro Savarese, Itay Evron, Daniel Soudry, and Nathan Srebro.
\newblock How do infinite width bounded norm networks look in function space?
\newblock In Alina Beygelzimer and Daniel Hsu, editors, \emph{Proceedings of
  the Thirty-Second Conference on Learning Theory}, volume~99 of
  \emph{Proceedings of Machine Learning Research}, pages 2667--2690, Phoenix,
  USA, 25--28 Jun 2019. PMLR.
\newblock URL \url{http://proceedings.mlr.press/v99/savarese19a.html}.

\bibitem[Sonoda and Murata(2013)]{sonoda2013nonparametric}
Sho Sonoda and Noboru Murata.
\newblock Nonparametric weight initialization of neural networks via integral
  representation.
\newblock \emph{arXiv preprint arXiv:1312.6461}, 2013.

\bibitem[Sonoda and Murata(2017)]{sonoda2017neural}
Sho Sonoda and Noboru Murata.
\newblock Neural network with unbounded activation functions is universal
  approximator.
\newblock \emph{Applied and Computational Harmonic Analysis}, 43\penalty0
  (2):\penalty0 233--268, 2017.

\end{thebibliography}

\appendix
\section{Proof of Theorem \ref{thm:main_multi_inv}}
\begin{lemma}\label{lem:dim_reduce}
For any $F \in L^1(\Sd)$ we have
\begin{align*}
    \int\limits_{\Sd}F(a,b)\dd\nu_d(a,b)
      & = \int\limits_0^\pi\sin^{d-1} \phi \int\limits_{\mathbb{S}^{d-1}} 
        F\left(\alpha \sin\phi, \cos\phi\right) \dd\nu_{d-1}(\alpha) \dd\phi\\
        &=\int\limits_{\mathbb{S}^{d-1}} \int\limits_\R\frac{1}{\sqrt{1+\beta^2}^{d+1}}
        F\bigg(\frac{\alpha}{\sqrt{1+\beta^2}}, \frac{\beta}{\sqrt{1+\beta^2}}\bigg) \dd\nu_{d-1}(\alpha) \dd\beta.
\end{align*}
\end{lemma}

\begin{proof}
Statement of the lemma is trivial for $d = 1$.
For $d > 1$ consider the following change of variables given by the spherical coordinates
\begin{align*}
	b &= \cos\phi_1,
	\\
    a_1 &= \sin\phi_1 \sin\phi_2 =: \sin\phi_1 \, \alpha_1,
    \\
    &\cdots
    \\
    a_{d-1} &= \sin\phi_1 \ldots\, \sin\phi_{d-1} \cos\phi_d =: \sin\phi_1 \, \alpha_{d-1},
    \\
	a_d &= \sin\phi_1 \ldots\, \sin\phi_{d-1} \sin\phi_d =: \sin\phi_1 \, \alpha_d,
\end{align*}
where $\phi_1,\dots,\phi_{d-1}\in[0,\pi]$ and $\phi_d\in [0,2\pi)$, and $\alpha=(\alpha_1,\dots,\alpha_d)\in\mathbb{S}^{d-1}$.
The area element on the unit sphere $\Sd$ is given by $\dd\nu_d(a,b) = \sin^{d-1}\phi_1 \sin^{d-2}\phi_{2} \ldots\, \sin\phi_{d-1} \dd\phi_1 \dd\phi_2 \ldots \dd\phi_{d-1}$.
Therefore we obtain
\begin{multline*}
    \int\limits_{\Sd}F(a,b)\dd\nu_d(a,b)
    \\
    = \int\limits_0^\pi \ldots \int\limits_0^\pi \int\limits_0^{2\pi} F\big(\alpha\sin\phi_1, \cos\phi_1\big)
        \sin^{d-1}\phi_1 \sin^{d-2}\phi_{2} \ldots \sin\phi_{d-1} \dd\phi_d \ldots\, \dd\phi_2 \dd\phi_1
    \\
    = \int\limits_0^\pi \sin^{d-1}\phi_1 \int\limits_0^\pi \ldots \int\limits_0^\pi \int\limits_0^{2\pi}
        F\big(\alpha\sin\phi_1, \cos\phi_1\big) \sin^{d-2}\phi_{2} \ldots \sin\phi_{d-1} \dd\phi_d \ldots\, \dd\phi_2 \dd\phi_1.
\end{multline*}
To complete the proof, we change the variable $\phi_1 \in [0,\pi]$ to $\beta = \cot\phi_1 \in \R$ to get $\cos\phi_1 = \nicefrac{\beta}{\sqrt{1+\beta^2}}$ and $\sin\phi_1 = \nicefrac{1}{\sqrt{1+\beta^2}}$.
Substituting into the above integral provides the required result.
\end{proof}

We also use the following technical result, which is a corollary of Proposition~8.1 from~\cite{kainen2010integral}.
\begin{lemma}\label{lem:reconstruction_reference}
Let $H :\R \to \R$ be the Heaviside function, i.e. $H(x) = \nicefrac{(1 + \operatorname{sgn}(x))}{2}$. Then for any compactly supported function $f$ in $C^d(\R^d)$ the following reconstruction formula holds:
\[
    f(x) = \left\{
        \begin{array}{ll}
            \displaystyle{
            -\frac{(-1)^{\nicefrac{d+1}{2}}}{2(2\pi)^{d-1}} \int\limits_{\mathbb{S}^{d-1}}\int\limits_\R
            \frac{\partial^{d}}{\partial \beta^{d}} \mathcal{R}[f](\alpha,\beta)
            \ H(\alpha \cdot x + \beta)\dd \beta\dd\nu_{d-1}(\alpha)}
            &\quad\text{if } d \text{ is odd};
            \\
            \displaystyle{
            -\frac{(-1)^{\nicefrac{d}{2}}}{2(2\pi)^{d-1}} \int\limits_{\mathbb{S}^{d-1}}\int\limits_\R
            \frac{\partial^{d}}{\partial \beta^{d}} \mathcal{H} \big[\mathcal{R}[f](\alpha,\beta) \big](\beta)
            \ H(\alpha \cdot x + \beta)\dd \beta\dd\nu_{d-1}(\alpha)}
            &\quad\text{if } d \text{ is even}.
        \end{array}
    \right.
\]
\end{lemma}

We now proceed to the proof of Theorem~\ref{thm:main_multi_inv}. Let $d$ be odd.
From  Lemma~\ref{lem:dim_reduce} we get
\begin{align*}
	\int\limits_{\Sd} \frac{1}{\|a\|^{d+2}}
    &\frac{\partial^{d+1}}{\partial b^{d+1}} \mathcal{R}[f](a,b) \ \sigma(a \cdot x + b) \dd\nu_d(a,b)
	\\
	&= \int\limits_{\mathbb{S}^{d-1}} \int\limits_\R \frac{\partial^{d+1}}{\partial \beta^{d+1}}
        \mathcal{R}[f] \bigg(\frac{\alpha}{\sqrt{1+\beta^2}}, \frac{\beta}{\sqrt{1+\beta^2}}\bigg)
        \, \sigma(\alpha \cdot x + \beta) \dd\beta\dd\nu_{d-1}(\alpha)
    \\
    &= \int\limits_{\mathbb{S}^{d-1}}\int\limits_\R
        \frac{\partial^{d+1}}{\partial \beta^{d+1}} R[f](\alpha,\beta) \ \sigma(\alpha \cdot x + \beta) \dd\beta\dd\nu_{d-1}(\alpha),
\end{align*}
where we use the positive homogeneity of Radon transform
\[
    \mathcal{R}[f]\bigg(\frac{\alpha}{\sqrt{1+\beta^2}}, \frac{\beta}{\sqrt{1+\beta^2}}\bigg)
    = \mathcal{R}[f]\left(\alpha,\beta\right).
\]
Then integration by parts provides
\[
    \int\limits_\R \frac{\partial^{d+1}}{\partial \beta^{d+1}} \mathcal{R}[f](\alpha,\beta) \ \sigma(\alpha \cdot x + \beta) \dd\beta
    = -\int\limits_\R \frac{\partial^{d}}{\partial \beta^{d}} \mathcal{R}[f](\alpha,\beta) \ H(\alpha \cdot x+\beta) \dd\beta
\]
and applying Lemma~\ref{lem:reconstruction_reference} completes the proof of this case.
The case of even $d$ is proven analogously.

\section{Proof of Theorem \ref{thm:main_inv}}
Before proving the theorem we perform several related calculations.
From the equality $\sigma(z) = \nicefrac{|z|}{2} + \nicefrac{z}{2}$ we obtain
\begin{align*}
    \int\limits_{0}^{2\pi} \cos\phi \ \sigma(x \cos\phi  +\sin\phi )\dd\phi
    &= \frac{1}{2} \int\limits_{0}^{2\pi} \cos\phi \ |x \cos\phi+\sin\phi |\dd\phi
    + \frac{1}{2} \int\limits_{0}^{2\pi} \cos\phi \ (x \cos\phi+\sin\phi )\dd\phi
    \\
    &= \frac{1}{2} \Bigg(\int\limits_{0}^{\pi} + \int\limits_\pi^{2\pi}\Bigg) \cos\phi \ |x \cos\phi+\sin\phi |\dd\phi
    + \frac{\pi x}{2} = \frac{\pi x}{2}
\end{align*}
and, in the same way,
\[
    \int\limits_{0}^{2\pi} \sin\phi \ \sigma(x \cos\phi  +\sin\phi )\dd\phi
    = \frac{\pi}{2}.
\]
From the mutual orthogonality of the functions $\{\sin\phi, \cos\phi, |\cos\phi|, s(\phi)\}$ we deduce 
\begin{align*}
    \int\limits_{0}^{2\pi} |\cos\phi| \, \sigma(x \cos\phi & +\sin\phi )\dd\phi
	\\
    &= \frac{1}{2} \int\limits_{0}^{2\pi} |\cos\phi | \, |x \cos\phi+\sin\phi |\dd\phi
    +\frac{1}{2} \int\limits_{0}^{2\pi} |\cos\phi| \, (x \cos\phi+\sin\phi )\dd\phi
    \\
    &= \int\limits_{-\nicefrac{\pi}{2}}^{\nicefrac{\pi}{2}} |\cos\phi |\,|x \cos\phi+\sin\phi |\dd\phi
    = \int\limits_{-\nicefrac{\pi}{2}}^{\nicefrac{\pi}{2}} \cos^2\phi \, |x+\tan\phi |\dd\phi
    \\
    &= \int\limits_{-\infty}^{\infty} \frac{|x+z|}{(1+z^2)^2}\dd z
    = x\arctan x + 1,
\end{align*}
where we changed the variable $\phi \in (-\nicefrac{\pi}{2},\nicefrac{\pi}{2})$ to $z = \tan\phi \in \R$.
Similarly, 
\begin{align*}
    \int\limits_{0}^{2\pi} s(\phi) \, \sigma(x \cos\phi  +\sin\phi )\dd\phi
    &= \frac{1}{2} \int\limits_{0}^{2\pi} s(\phi) \, |x \cos\phi+\sin\phi|\dd\phi
    + \frac{1}{2} \int\limits_{0}^{2\pi} s(\phi) \, (x \cos\phi+\sin\phi )\dd\phi
    \\
    &= \frac{1}{2} \int\limits_{0}^{2\pi} s(\phi) \, |x \cos\phi+\sin\phi |\dd\phi
    = \int\limits_{-\frac{\pi}{2}}^{\frac{\pi}{2}} \tan\phi \, \cos^2\phi \, |x+\tan\phi |\dd\phi
    \\
    &= \int\limits_{-\infty}^{\infty} \frac{z \, |x+z|}{(1+z^2)^2} \dd z
    = \arctan x.
\end{align*}

We now prove the direct implication.
Assume that a function $f$ admits the integral representation
\[
    f(x) = \int\limits_{0}^{2\pi} c(\phi) \, \sigma(x \cos\phi+\sin\phi )\dd\phi
\]
with a weight function $c\in L_1[0,2\pi]$.
Note that due to the mutual orthogonality of the functions $\{\sin\phi ,\cos\phi ,|\cos\phi |,s(\phi)\}$ we can assume without loss of generality that $\alpha = \beta = \gamma = \eta = 0$ by replacing the weight function $c(\phi)$ with
\[
    c(\phi) - \alpha \cos\phi - \beta \sin\phi - \gamma |\cos\phi| - \eta \, s(\phi).
\]
Then from the mutual orthogonality of $\{c(\phi), \sin\phi ,\cos\phi ,|\cos\phi |,s(\phi)\}$ we deduce
\begin{equation}\label{cond:orthtotrig}
	\int\limits_{0}^{\pi} c(\phi) \cos\phi \dd \phi
	= \int\limits_{0}^{\pi} c(\phi+\pi) \cos\phi \dd \phi
	= \int\limits_{0}^{\pi} c(\phi) \sin\phi \dd \phi
	= \int\limits_{0}^{\pi} c(\phi+\pi) \sin\phi \dd \phi
	= 0.
\end{equation}

We will show that $f$ is in the class $\mathcal{W}(\R)$.
First, we show that $\lim\limits_{x \to \pm\infty} f(x) = 0$.
Indeed, from condition~\eqref{cond:orthtotrig} we get for any $x > 0$
\begin{align*}
    f(x)
    &= \int\limits_{-\nicefrac{\pi}{2}}^{\nicefrac{\pi}{2}} c(\phi) \, \sigma(x \cos\phi+\sin\phi )\dd\phi
    + \int\limits_{-\nicefrac{\pi}{2}}^{\nicefrac{\pi}{2}} c(\phi+\pi) \, \sigma(-x \cos\phi-\sin\phi )\dd\phi
    \\
    &= \int\limits_{-\arctan x}^{\nicefrac{\pi}{2}} c(\phi) \, (x \cos\phi+\sin\phi )\dd\phi
    + \int\limits_{-\nicefrac{\pi}{2}}^{-\arctan x} c(\phi+\pi) \, (-x \cos\phi-\sin\phi )\dd\phi
    \\
%    &= -\int\limits_{-\nicefrac{\pi}{2}}^{-\arctan x} c(\phi) \, (x \cos\phi+\sin\phi )\dd\phi
%    + \int\limits_{-\nicefrac{\pi}{2}}^{-\arctan x} c(\phi+\pi) \, (-x \cos\phi-\sin\phi )\dd\phi
%    \\
    &= -x \int\limits_{-\nicefrac{\pi}{2}}^{-\arctan x} \big(c(\phi) + c(\phi+\pi)\big) \cos\phi \dd\phi
    - \int\limits_{-\nicefrac{\pi}{2}}^{-\arctan x} \big(c(\phi) + c(\phi+\pi)\big) \sin\phi \dd\phi.
\end{align*}
By using the relation $\cos(\arctan x) = \nicefrac{1}{\sqrt{1+x^2}}$ we obtain
\begin{equation}\label{eq:int_estimate}
    \left|\, x \int\limits_{-\nicefrac{\pi}{2}}^{-\arctan x} \big(c(\phi) + c(\phi+\pi)\big) \cos\phi \dd\phi \,\right|
    \le \frac{x}{\sqrt{1 + x^2}} \int\limits_{-\nicefrac{\pi}{2}}^{-\arctan x} \big|c(\phi) + c(\phi+\pi)\big| \dd\phi.
\end{equation}
Then from condition~\eqref{cond:orthtotrig} we get $\lim\limits_{x \to \infty} f(x) = 0$.
By a similar argument we have $\lim\limits_{x \to -\infty} f(x) = 0$.

Next, we show the existence of the derivative $f^\prime$ and that $\lim\limits_{x \to \pm\infty} x f^\prime(x) = 0$.
Let $H$ denote the Heaviside step function, then by using dominated convergence theorem we get for any $x > 0$
\begin{align}
    \nonumber
    f^\prime(x)
    &= \int\limits_{0}^{2\pi} c(\phi) \, H(x\cos\phi+\sin\phi) \cos\phi \dd\phi
    \\
    \nonumber
    &= \int\limits_{-\nicefrac{\pi}{2}}^{\nicefrac{\pi}{2}} c(\phi) \, H(x \cos\phi+\sin\phi) \cos\phi \dd\phi
    - \int\limits_{-\nicefrac{\pi}{2}}^{\nicefrac{\pi}{2}} c(\phi+\pi) \, H(-x \cos\phi-\sin\phi) \cos\phi \dd\phi
    \\
    \nonumber
    &= \int\limits_{-\arctan x}^{\nicefrac{\pi}{2}} c(\phi) \cos\phi \dd\phi
    - \int\limits_{-\nicefrac{\pi}{2}}^{-\arctan x} c(\phi+\pi) \cos\phi \dd\phi
    \\
    \label{eq:f_firstder}
    &= -\int\limits_{-\nicefrac{\pi}{2}}^{-\arctan x} \big( c(\phi) + c(\phi + \pi) \big) \cos\phi \dd\phi.
\end{align}
By taking into account estimate~\eqref{eq:int_estimate} we derive $\lim\limits_{x\to \infty} x f^\prime(x) = 0$.
Condition $\lim\limits_{x\to -\infty} x f^\prime(x) = 0$ proves in a similar way.

Next we show that the second derivative $f^{\prime\prime}$ exists almost everywhere.
Indeed, from~\eqref{eq:f_firstder} we see that $f^{\prime\prime}(x)$ exists at every $x$ such that $-\arctan x$ is a Lebesgue point of $\big( c(\phi) +c(\phi+\pi) \big) \cos\phi$, which is almost everywhere since $c \in L_1[0,2\pi]$.
In that case we get
\begin{align}
    \nonumber
    f^{\prime\prime}(x)
    &= \frac{1}{1+x^2} \, \big( c(-\arctan x) + c(-\arctan x + \pi) \big) \cos(-\arctan x)
    \\
    \label{eq:f_secder}
    &= \frac{c(-\arctan x) + c(-\arctan x + \pi)}{(1+x^2)^{\nicefrac{3}{2}}}.
\end{align}
Finally, by changing the variable from $x \in \R$ to $\phi = -\arctan x \in (-\nicefrac{\pi}{2}, \nicefrac{\pi}{2})$, we estimate
\begin{align*}
    \int\limits_{-\infty}^\infty \big| f^{\prime\prime}(x) \big| \sqrt{1+x^2} \dd x
    &= \int\limits_{-\infty}^\infty \frac{1}{1+x^2} \, \big| c(-\arctan x) + c(-\arctan x + \pi) \big| \dd x
    \\
    &= \int\limits_{-\nicefrac{\pi}{2}}^{\nicefrac{\pi}{2}} \big| c(\phi) + c(\phi+\pi) \big| \dd\phi
    \le \int\limits_{0}^{2\pi} |c(\phi)| \dd\phi < \infty.
\end{align*}
Therefore $f \in \mathcal{W}(\R)$.

Lastly, we show that the weight function $c$ has the form~\eqref{eq:weight_function_form} with some $k \in \mathcal{W}(\R)$.
Denote
\[
    \bar c(\phi) =
    \left\{\begin{array}{ll}
        \phantom{-}c(\phi), & \phi \in [-\nicefrac{\pi}{2}, \nicefrac{\pi}{2}),
        \\
        -c(\phi), & \phi \in [\nicefrac{\pi}{2}, \nicefrac{3\pi}{2}).
    \end{array}\right.
\]
Then $\bar{c} \in L_1[0,2\pi]$ and satisfies conditions~\eqref{cond:orthtotrig}, hence
\[
    k(x) := \int\limits_0^{2\pi} \bar{c}(\phi) \, \sigma(x\cos\phi + \sin\phi) \dd\phi \in \mathcal{W}(\R).
\]
From \eqref{eq:f_secder} we deduce that for almost all $\phi \in [0, 2\pi]$
\begin{align*}
    k^{\prime\prime}(-\tan\phi)
    &= \big(\bar{c}(\phi) + \bar{c}(\phi+\pi)\big) \cos^3\phi
    = \big(c(\phi) - c(\phi+\pi)\big) \cos^3\phi,
    \\
    f^{\prime\prime}(-\tan\phi)
    &= \big(c(\phi) + c(\phi+\pi)\big) |\cos^3\phi|.
\end{align*}
Combining these relations we conclude 
\[
    c(\phi) = \frac{f^{\prime\prime}(-\tan\phi)}{2|\cos^3\phi|} + \frac{k^{\prime\prime}(-\tan\phi)}{2\cos^3\phi}.
\]
Since we initially subtracted the term $\alpha\cos\phi + \beta\sin\phi + \gamma\,|\cos(x)| + \eta\,s(\phi)$ from the weight function $c(\phi)$, in a general case we will have
\[
    c(\phi) = \frac{f^{\prime\prime}(-\tan\phi)}{2|\cos^3\phi|} + \frac{k^{\prime\prime}(-\tan\phi)}{2\cos\phi^3}
    + \alpha\cos\phi + \beta\sin\phi + \gamma\,|\cos\phi| + \eta\,s(\phi),
\]
which completes the proof of the direct implication.

\bigskip
We now prove the inverse implication.
Assume that function $f$ has the form
\[
    f(x) = g(x) + \alpha x + \beta + \gamma(x\arctan x + 1) + \eta \arctan x
\]
with some function $g \in \mathcal{W}(\R)$ and constants $\alpha, \beta, \gamma, \eta \in \R$.
Similarly to the direct case, we can assume that $\alpha = \beta = \gamma = \eta = 0$ by replacing function $f(x)$ with
\[
    f(x) - \alpha x - \beta - \gamma(x\arctan x + 1) - \eta \arctan x.
\]
Denote
\[
    c(\phi) = \frac{f^{\prime\prime}(-\tan\phi)}{2| \cos^3\phi|} + \frac{g^{\prime\prime}(-\tan\phi)}{2\cos^3\phi}.
\]
We will show that $c \in L_1[0,2\pi]$ and that $\int_0^{2\pi} c(\phi) \, \sigma(x\cos\phi+\sin\phi) \dd\phi = f(x)$.
First, note that
\begin{align*}
    \int\limits_{0}^{2\pi} |c(\phi)| \dd\phi
    &\le 2\int\limits_{-\frac{\pi}{2}}^{\frac{\pi}{2}} \frac{|f^{\prime\prime}(-\tan\phi )|}{2|\cos\phi |^3} \dd \phi
    + 2\int\limits_{-\frac{\pi}{2}}^{\frac{\pi}{2}} \frac{|g^{\prime\prime}(-\tan\phi )|}{2|\cos\phi |^3} \dd \phi
    \\ 
    &= \int\limits_{-\infty}^\infty|f^{\prime\prime}(z)|(1+z^2)^{\frac{1}{2}} \dd z
    + \int\limits_{-\infty}^\infty|g^{\prime\prime}(z)|(1+z^2)^{\frac{1}{2}} \dd z
    < \infty
\end{align*}
and hence $c \in L_1[0,2\pi]$.
Taking into account that $\sigma(z) + \sigma(-z) = |z|$ and using the assumption $f \in \mathcal{W}(\mathbb{R})$ we get
\begin{align*}
    \int\limits_{0}^{2\pi} \frac{f^{\prime\prime}(-\tan\phi )}{2|\cos^3\phi|} \, &\sigma(x \cos\phi+\sin\phi ) \dd\phi
    \\
    &= \int\limits_{-\nicefrac{\pi}{2}}^{\nicefrac{\pi}{2}} \frac{f^{\prime\prime}(-\tan\phi)}{2\cos^2\phi} \, \sigma(x+\tan\phi ) \dd \phi
    + \int\limits_{\nicefrac{\pi}{2}}^{\nicefrac{3\pi}{2}} \frac{f^{\prime\prime}(-\tan\phi)}{2\cos^2\phi} \, \sigma(-x-\tan\phi )\dd \phi
    \\
    &= \frac{1}{2}\int\limits_{-\infty}^\infty f^{\prime\prime}(z) |x-z| \dd z
    = -\frac{1}{2}\int\limits_{-\infty}^\infty f^{\prime}(z) \, \operatorname{sign}(x-z) \dd z
    = f(x).
\end{align*}
On the other hand, since $\sigma(z) - \sigma(-z) = z$ and from the assumption $g \in \mathcal{W}(\mathbb{R})$ we obtain
\begin{align*}
    \int\limits_{0}^{2\pi} \frac{g^{\prime\prime}(-\tan\phi )}{2\cos^3\phi} \, &\sigma(x \cos\phi+\sin\phi ) \dd\phi
    \\
    &= \int\limits_{-\nicefrac{\pi}{2}}^{\nicefrac{\pi}{2}} \frac{g^{\prime\prime}(-\tan\phi)}{2\cos^2\phi} \, \sigma(x+\tan\phi ) \dd \phi
    - \int\limits_{\nicefrac{\pi}{2}}^{\nicefrac{3\pi}{2}} \frac{g^{\prime\prime}(-\tan\phi)}{2\cos^2\phi} \, \sigma(-x-\tan\phi )\dd \phi
    \\
    &= \frac{1}{2}\int\limits_{-\infty}^\infty g^{\prime\prime}(z) (x-z) \dd z
    = \frac{1}{2}\int\limits_{-\infty}^\infty g^{\prime}(z) \dd z
    = 0.
\end{align*}
Hence $\int_0^{2\pi} c(\phi) \, \sigma(x\cos\phi+\sin\phi) \dd\phi = f(x)$, which completes the proof.

\section{Proof of Theorem \ref{thm:least_l1}}
From Theorem~\ref{thm:main_inv} we deduce that any weight function $c(\phi)$ satisfying
\[
    \int\limits_{0}^{2\pi} c(\phi) \, \sigma(x \cos\phi+\sin\phi )\dd\phi = f(x)
    \]
has the form
\[
    c(\phi) = \frac{f^{\prime\prime}(-\tan\phi)}{2|\cos^3\phi|} + \frac{g^{\prime\prime}(-\tan\phi)}{2\cos^3\phi}
\]
with some $g \in \mathcal{W}(\R)$.
Hence
\begin{align*}
    \int\limits_0^{2\pi} |c(\phi)| \dd \phi
    &= \frac{1}{2}\int\limits_{-\infty}^\infty |f^{\prime\prime}(x) + g^{\prime\prime}(x)| \sqrt{1+x^2} \dd x 
    + \frac{1}{2}\int\limits_{-\infty}^\infty |f^{\prime\prime}(x) - g^{\prime\prime}(x)| \sqrt{1+x^2} \dd x
    \\
    &\ge \int\limits_{-\infty}^\infty| f^{\prime\prime}(x)| \sqrt{1+x^2} \dd x
\end{align*}
and the minimal value of $\|c\|_1$ is attained at $g\equiv 0$.

\end{document}